\documentclass[letterpaper, 11pt]{article}

\usepackage[english]{babel}
\usepackage[utf8x]{inputenc}
\usepackage[T1]{fontenc}
\usepackage{chngcntr}
\usepackage{apptools}
\AtAppendix{\counterwithin{lemma}{section}}
\usepackage[margin=1in]{geometry}
\usepackage{graphicx}
\usepackage{subfigure}
\usepackage{booktabs} 
\usepackage{amsfonts}       
\usepackage{authblk}
\usepackage{amsmath}
\usepackage{empheq}
\usepackage{amssymb}
\usepackage{amsthm}
\usepackage{comment}
\usepackage{bbm}
\usepackage{algorithm}
\usepackage{algorithmic}
\newcommand{\E}{\mathbb{E}}
\newtheorem{theorem}{Theorem}
\newtheorem{lemma}{Lemma}
\newtheorem{corollary}{Corollary}

\newtheorem{definition}{Definition}
\usepackage{graphicx}
\newcommand{\D}{\mathcal{D}}

\newcommand{\cov}{\mathsf{Cov}}
\newcommand{\avg}{\mathsf{avg}}

\renewcommand{\hat}{\widehat}
\DeclareMathOperator*{\argmax}{arg\,max}
\allowdisplaybreaks
\usepackage{hyperref}

\title{Learning Restricted Boltzmann Machines with Arbitrary External Fields}

\author[]{Surbhi Goel\footnote{\fontfamily{qcr}\selectfont surbhi@cs.utexas.edu}}
\affil[]{University of Texas at Austin}
\date{\today}

\begin{document}

\maketitle
\begin{abstract}
We study the problem of learning graphical models with latent variables. We give the first algorithm for learning locally consistent (ferromagnetic or antiferromagnetic) Restricted Boltzmann Machines (or RBMs) with {\em arbitrary} external fields. Our algorithm has optimal dependence on dimension in the sample complexity and run time however it suffers from a sub-optimal dependency on the underlying parameters of the RBM.

Prior results have been established only for {\em ferromagnetic} RBMs with {\em consistent} external fields (signs must be same)\cite{bresler2018learning}. The proposed algorithm strongly relies on the concavity of magnetization which does not hold in our setting. We show the following key structural property: even in the presence of arbitrary external field, for any two observed nodes that share a common latent neighbor, the covariance is high. This enables us to design a simple greedy algorithm that maximizes covariance to iteratively build the neighborhood of each vertex.
\end{abstract}

\section{Introduction}
Graphical models are a popular framework for expressing high dimensional distributions by using an underlying graph to represent conditional dependencies among the variables. Learning the underlying dependency structure of a graphical model using samples drawn from the distribution is a core problem in understanding graphical models. Much progress has been made in the recent years towards developing efficient algorithms for learning fundamental models such as Ising model and Markov random fields (MRFs) with near optimal sample and time complexity under the assumptions of sparsity and/or correlation decay.

The {\em structure learning} problem becomes even more challenging when the underlying model is allowed to have latent (or hidden) variables. Compared to fully observed models, latent variable models can induce more complex dependencies among the observed variables once the latent variables are marginalized. In this work we restrict ourselves to a special class of latent variable models known as Restricted Boltzmann machines (RBMs). RBMs have been used for various unsupervised learning tasks \cite{hinton2006reducing,larochelle2008classification,salakhutdinov2007restricted,hinton2009replicated} since their inception in the early 2000s by Geoffrey Hinton. In RBMs, the interactions are restricted to be pairwise between observed and latent variables. More formally, a RBM induces a probability distribution over $n$ observed variables $X \in \{\pm 1\}^n$ and $m$ latent variables $Y \in \{\pm 1\}^m$ as follows,
\[
\Pr[X= x, Y= y] = \frac{1}{Z} \exp(x^TJy + h^Tx + g^Ty)
\]
Here $J \in \mathbb{R}^{n \times m}$ is the interaction matrix, $h \in \mathbb{R}^n, g \in \mathbb{R}^m$ are the external fields and $Z$ is the partition function. Alternatively, a RBM can be viewed as a bipartite graph between the set of observed and latent variables with edge weights given by $J$. 

Recently Bresler et. al. \cite{bresler2018learning} proposed an algorithm that learns {\em ferromagnetic} RBMs ($J \ge 0$) with non-negative external fields ($h,g \ge 0$). They apply the famous Griffiths-Hurst-Sherman correlation inequality to prove that a certain influence function is submodular and use a simple greedy algorithm to maximize the same. Their work relies heavily on the GHS inequality which requires the external fields to be {\em consistent}, that is, have the same sign.

In this paper we focus on learning {\em locally consistent} RBMs (outgoing edges of each latent variable have the same sign) with {\em arbitrary} external fields. The presence of inconsistent external fields allows for different biases on different hidden nodes potentially creating more conflicts between the observed nodes making the problem more challenging. It is well-known that the presence of arbitrary external fields can greatly change the complexity of closely related problems such as approximating the partition function \cite{goldberg2007complexity}.

\paragraph{Our Results.} 
The main contribution of our paper is the following key structural property of locally consistent RBMs with arbitrary external fields.
\begin{lemma}[Informal version of Lemma \ref{lem:pos}]
For any observed node $u$ in a locally consistent RBM, for all observed nodes $v$ that share a common neighbor with $u$ in the underlying graph, the covariance between $u$ and $v$ is at least some positive constant independent of the dimension $n$.
\end{lemma}
The above key property gives us the following structure learning result for locally consistent RBMs.
\begin{theorem}[Informal version of Theorem \ref{thm:main}]
Consider a locally consistent RBM with arbitrary external fields such that all non-zero interactions are bounded below by $\alpha$ and the sum of absolute weights of outgoing edges of every node (plus absolute value of external field) is bounded above by $\lambda$ then there is an algorithm that recovers the markov blanket of each observed variable in time $\widetilde{O}_{\alpha, \lambda}(n^2)$ and sample complexity $O_{\alpha, \lambda}(\log n)$\footnote{The sub-script indicates that the dependency on $\alpha, \lambda$ is suppressed. Also $\widetilde{O}$ hides logarithmic dependencies.}. 
\end{theorem}
Here the dependence on $\alpha$ is exponential and that on $\lambda$ is doubly exponential. Singly exponential dependence is necessary for learning. Note that our bounds are similar to those in \cite{bresler2015efficiently}. However, we note that for ferromagnetic RBMs with consistent fields, \cite{bresler2018learning} have a singly exponential dependence on $\lambda$ which is optimal. The question to improve the dependence on $\alpha,\lambda$ for locally consisitent RBMs with arbitrary external fields is an outstanding open question.

\paragraph{Our Techniques.} For our key structural result, we define a transformation on the variables that enables us to use symmetry arguments in order to prove the non-negativity of the covariance. A more involved analysis lets us go further and bound the covariance by a constant independent of the input dimension. 

For learning RBMs, in the spirit of the influence maximization algorithm due to Bresler \cite{bresler2015efficiently}, we maximize covariance to iteratively build the neighborhood of each observed vertex. Using an entropy argument, we can show that our iterative algorithm returns us the exact neighborhood of each vertex.

\paragraph{Related Work.} Structure learning for graphical models is a well studied problem, with major focus on the fully-observed model. The first algorithms were proposed by Chow and Liu \cite{chow1968approximating} for learning undirected graphical models on trees. Subsequently, various algorithms were proposed for structure learning under varying assumptions on the underlying model \cite{lee2007efficient,ravikumar2010high,yang2012graphical,bresler2015efficiently,vuffray2016interaction,klivans2017learning,hamilton2017information,wu2018sparse}. Bresler \cite{bresler2015efficiently} proposed a simple greedy algorithm based on influence maximization for assumption-free structure learning of Ising models. His algorithm achieved optimal sample/time complexity in terms of the dimension however depended doubly exponentially on the degree of the underlying graph. Subsequently Vuffray et. al. \cite{vuffray2016interaction} and Klivans and Meka \cite{klivans2017learning} proposed alternative techniques to remove the doubly exponential dependence. 

The problem of structure recovery in the presence of latent variables is not as well understood as the fully-observed setting. For locally tree-like models, Anandkumar and Valluvan \cite{anandkumar2013learning} gave efficient algorithms for recovery under correlation decay assumption. Assuming that the latent variables are distributed according to a Gaussian distribution, Nussbaum and Giesen \cite{nussbaum2019ising} proposed a likelihood model for sparse + low rank model for stucture learning. The most relevant to our work is that of \cite{bresler2018learning} which proposed the first algorithm to recover the structure of ferromagnetic RBMs with non-negative external fields using concavity of magnetization. Unlike their setup, we allow the external fields to be arbitrary and relax the ferromagnetic condition to a locally-consistent condition at the cost of a worse dependence on $\alpha, \lambda$.

\section{Preliminaries}
We consider a RBM on underlying bipartite graph $G = (V_{obs}, V_{lat}, E)$ over observed variables $X$ and latent variables $Y$ with $|V_{obs}| = n$ and $|V_{lat}|=m$. Each configuration of observed/latent variables $\in \pm 1$ is assigned probability
\[
\Pr[X= x, Y= y] = \frac{1}{Z} \exp(x^TJy + h^Tx + g^Ty)
\]
where $J$ is the interaction matrix and $h,g$ are external fields. In this work, we consider the following class of locally consistent RBMs.
\begin{definition}
A RBM is said to be $(\alpha, \lambda)$-locally consistent if the following conditions are satisfied:
\begin{itemize}
    \item $J$ is locally consistent, that is, for each $j \in [m]$, $J_{ij} \geq 0$ for all $i$ (ferromagnetic) or $J_{ij} \leq 0$ for all $i$ (anti-ferromagnetic).
    \item For all $(i,j) \in E$ such that $|J_{ij}| \ge \alpha$.
    \item For all $i \in [n]$, $\sum_{j} |J_{ij}| + |h_i| \leq \lambda$.
    \item For all $j \in [m]$, $\sum_{i} |J_{ij}| + |g_j| \leq \lambda$.
\end{itemize}
\end{definition}
Define $N(u):= \{j: J_{uj} \ne 0\}$ to be the graph-theoretic neighborhood of observed node $u$ and define $N_2(u) = \{i: \exists~j, J_{ij},J_{uj} \ne 0\}$ to be the two-hop graph-theoretic neighborhood. We also define $N^{mkv}_2(u)$ to be the two-hop Markov neighborhood, that is, the smallest set $S \subseteq V_{obs} \backslash \{u\}$ such that conditioned on $X_S$, $X_u$ is independent of $X_v$ for all $v \in V_{obs} \backslash (S \cup \{u\})$. 

Our objective is to recover the two-hop Markov neighborhood of each observed variable. In our setting, this will correspond to the two-hop graph-theoretic neighborhood of each observed variables.

\paragraph{Remark.} We can WLOG assume $J \ge 0$ since if there exists $j$ such that $J_{ij} \le 0$ for all $i$ (locally consistent) then we can map $Y_j \rightarrow -Y_j$ without affecting the marginal on $X$ and the model is ferromagnetic at $j$. The change of variable will reverse the external field at $j$ however since we do not make any assumption on the sign of the external field, our model assumptions still hold. We can repeat this for all such $j$ and the model can therefore be made globally ferromagnetic. We will subsequently assume that $J\ge 0$.

\section{Conditional Covariance}
In this section we present our main structural result. We show that for two observed nodes sharing a common latent neighbor, the covariance is positive and bounded away from 0. The main motivation to believe that such a structural result holds is the famous FKG inequality \cite{percus1975correlation,sylvester1976inequalities} which states that for ferromagnetic Ising models with arbitrary external field the covariance of any two nodes is non-negative.

Define the conditional covariance for observed nodes $u, v \in V_{obs}$ and a subset of observed nodes $S \subseteq V_{obs} \backslash \{u,v\}$ with configuration $x_S$ as follows,
\[
\cov(u,v|X_S = x_S):= \E[X_uX_v| X_S = x_S] - \E[X_u|X_S = x_S]~\E[X_v | X_S =x_S].
\]
We also define the notion of average conditional covariance as follows, $\cov^{\avg}(u,v|S) = \E_{x_S}[\cov(u,v|X_S = x_S)]$. We will prove the following useful property of the conditional covariance:
\begin{lemma}\label{lem:cov}
For fixed node $u$ and any fixed subset of observed nodes $S \subseteq V_{obs} \backslash \{u\}$ with configuration $x_S$, then for all $v \in N_2(u) \backslash S$,
\[
\cov(u,v| X_S = x_S) \ge \alpha^2\exp(-12\lambda).
\]
\end{lemma}
\begin{proof}
It is easy to verify that on conditioning over a set of observed variables ($X_{S} = x_{S}$)), an $(\alpha, \lambda)$-locally consistent RBM remains an $(\alpha, \lambda)$-locally consistent RBM. Moreover, the edges between the the remaining nodes remain the same with the same edge weights. Thus, we can restrict to looking at $S = \emptyset$. Also, we will WLOG assume $J\ge 0$ as discussed before.

Consider the direct sum of two RBM $G \oplus G$ with two copies of $G$ such that the probability of a configuration under this new distribution $\D$ is 
\[
\Pr[X= x, Y= y, X' = x', Y' = y'] \propto \exp(x^TJy + h^Tx + g^Ty + x'^TJy' + h^Tx' + g^Ty')
\]
Define $X^{-}_i = \frac{X_i - X'_i}{\sqrt{2}}, Y^{-}_i = \frac{Y_i - Y'_i}{\sqrt{2}}$ and $X^{+}_i = \frac{X_i + X'_i}{\sqrt{2}}, Y^{+}_i = \frac{Y_i + Y'_i}{\sqrt{2}}$. Then we have
\begin{align*}
&\Pr[X= x, Y= y, X' = x', Y' = y']\\
& \propto \exp(x^TJy + h^Tx + g^Ty + x'^TJy' + h^Tx' + g^Ty')\\
&= \exp\left(\frac{1}{2}(x^TJy + x'^TJy + x^TJy' + x'^TJy') + \frac{1}{2}(x^TJy - x'^TJy - x^TJy' + x'^TJy')\right.  \\
&\quad \left.+ h^T(x + x') + g^T(y + y')\right)\\
&= \exp\left((x^{+})^TJy^{+} + (x^{-})^TJy^{-} + \sqrt{2}h^Tx^{+} + \sqrt{2}g^Ty^{+}\right).
\end{align*}
Observe that $\Pr[X= x, Y= y, X' = x', Y' = y'] = \Pr[X= x', Y= y', X' = x, Y' = y] = \Pr[X= x, Y= y]\Pr[X' = x', Y' = y']$. Thus under this transformation, we have
\begin{align*}
\cov(u,v) &= \E[X_uX_v] - \E[X_u]\E[X_v] \\
&= \E_{\D}[X_u X_v] - \E_{\D}[X_uX'_v] = \E_{\D}[X'_u X'_v] - \E_{\D}[X'_uX_v]\\
&= \frac{1}{2}(\E_{\D}[X_u X_v] + \E_{\D}[X'_u X'_v] - \E_{\D}[X'_u X_v] - \E_{\D}[X_uX'_v])\\
&= \frac{1}{2}\E_{\D}[(X_u - X'_u)(X_v - V'_v)]\\
&= \E_{\D}[X^{-}_u X^{-}_v]\\
&= \frac{\sum\limits_{\substack{x,x' \in \{\pm 1\}^n\\y,y' \in \{\pm 1\}^m}}x^{-}_ux^{-}_v \exp((x^{+})^TJy^{+} + (x^{-})^TJy^{-} + \sqrt{2}h^Tx^{+} + \sqrt{2}g^Ty^{+})}{\sum\limits_{\substack{x,x' \in \{\pm 1\}^n\\y,y' \in \{\pm 1\}^m}} \exp((x^{+})^TJy^{+} + (x^{-})^TJy^{-} + \sqrt{2}h^Tx^{+} + \sqrt{2}g^Ty^{+})}.
\end{align*}
Now we will bound the numerator (N) and denominator (D) separately. Since $v \in N_2(u)$, there exists $k$ such that $J_{uk}, J_{vk} \neq 0$. Let $\gamma(x^{-}, y^{-}) = \exp((x^{-})^TJy^{-} - x^-_uJ_{uk}y^-_{k} - x^-_vJ_{vk}y^-_k)$ and $\Delta(x^+, y^+) = \exp((x^{+})^TJy^{+} +\sqrt{2}h^Tx^{+} + \sqrt{2}g^Ty^{+})$. 
We have,
\begin{align*}
N &=\sum\limits_{\substack{x,x' \in \{\pm 1\}^n\\y,y' \in \{\pm 1\}^m}}x^{-}_ux^{-}_v \exp((x^{+})^TJy^{+} + (x^{-})^TJy^{-} + \sqrt{2}h^Tx^{+} + \sqrt{2}g^Ty^{+})\\
&=\sum\limits_{\substack{x,x' \in \{\pm 1\}^n\\y,y' \in \{\pm 1\}^m}}x^{-}_ux^{-}_v\exp((x^{-}_uJ_{uk} + x^{-}_vJ_{vk})y^{-}_k) \gamma(x^{-}, y^-)\Delta(x^+, y^+)\\
&=\sum\limits_{\substack{x,x' \in \{\pm 1\}^n\\y,y' \in \{\pm 1\}^m}} \sum_{i=0}^\infty x^{-}_ux^{-}_v\left(\frac{(x^{-}_uJ_{uk} + x^{-}_vJ_{vk})^i(y^{-}_k)^i}{i!}\right)\gamma(x^{-}, y^-)\Delta(x^+, y^+)\\
&=\sum\limits_{\substack{x,x' \in \{\pm 1\}^n\\y,y' \in \{\pm 1\}^m}}\sum_{i=0}^\infty \sum_{j=0}^i \frac{1}{i!}  {i \choose j}  J_{uk}^j J_{vk}^{i - j} (x^{-}_u)^{j+1}(x^{-}_v)^{i + 1 - j}(y^{-}_k)^i \gamma(x^{-}, y^-)\Delta(x^+, y^+)
\end{align*}
The following lemma is the main observation to bound the above term, it shows that each term in the summation is non-negative.
\begin{lemma}\label{lem:pos}
For all $A \in \mathbb{Z}_+^n, B \in \mathbb{Z}_+^n$ and function $f$ over $x^+, y^+$ such that $f \ge 0$,
\[
\sum\limits_{\substack{x,x' \in \{\pm 1\}^n\\y,y' \in \{\pm 1\}^m}}\prod_{a \in [n]}(x^-_a)^{A_a} \prod_{b \in [m]}(y^-_b)^{B_b} f(x^+, y^+) \ge 0.
\]
\end{lemma}
\begin{proof}
Observe that for any $i \in [n]$,  exchanging $x_i \leftrightarrow x'_i$ does not change the summation, however it changes $x^-_i \rightarrow -x^-_i$ while leaving $x^+_i \rightarrow x^+_i$ unchanged. Thus, if $A_i$ is odd, then the summation will be 0. Therefore, for the term to be non-zero, for all $i \in [n]$, $A_i$ must be even. Similarly, for all $j \in [m]$, $B_j$ must be even. Now since $f \ge 0$ and there are only even powers, the summation must be positive.
\end{proof}
It is easy to see that $\gamma(x^-, y^-)$ can be expanded as a multivariate polynomial over $x^-, y^-$ with non-negative coefficients (since $J \ge 0$)\footnote{Since $\gamma$ is an exponential function of a polynomial with non-negative coefficients, using taylor expansion of $e^{a}$, we will overall get a polynomial with all non-negative coefficients.}. Therefore, applying Lemma \ref{lem:pos}, we have for all $i \ge j$,
\[\sum_{x,x' \in \{\pm 1\}^n; y,y' \in \{\pm 1\}^m}(x^{-}_u)^{j+1}(x^{-}_v)^{i + 1 - j}(y^{-}_k)^i \gamma(x^{-}, y^-)\Delta(x^+, y^+) \ge 0.
\]
This implies that the covariance is indeed non-negative.

Now we will show that in fact the covariance is at least a constant independent of $n$. Since all terms are non-negative, we can lower bound the numerator by the term corresponding to $i = 2$ and $j = 1$. This yields only squares of $x^-_u, x^-_v, y^-_k$ as follows,
\begin{align*}
N &\ge \sum\limits_{\substack{x,x' \in \{\pm 1\}^n\\y,y' \in \{\pm 1\}^m}} J_{uk}J_{vk}(x^{-}_u)^{2}(x^{-}_v)^{2}(y^{-}_k)^{2} \gamma(x^{-}, y^-)\Delta(x^+, y^+)\\
&\ge \alpha^2\sum\limits_{\substack{x,x' \in \{\pm 1\}^n\\y,y' \in \{\pm 1\}^m}}(x^{-}_u)^{2}(x^{-}_v)^{2}(y^{-}_k)^{2} \gamma(x^{-}, y^-)\Delta(x^+, y^+).
\end{align*}
Here the second inequality follows from noting that by our assumption $J_{uk}, J_{vk} \ne 0$ and hence must be at least $\alpha$. Lastly we bound $\gamma(x^{-}, y^-)\Delta(x^+, y^+)$. Define $L:= V_{obs} \backslash \{u,v\}$ and $R:= V_{lat} \backslash \{k\}$.  We have
\begin{align*}
&\gamma(x^{-}, y^-)\Delta(x^+, y^+)\\
& =\exp((x^{-})^TJy^{-} -x^-_uJ_{uk}y^-_{k} - x^-_vJ_{vk}y^-_k + (x^{+})^TJy^{+} +\sqrt{2}h^Tx^{+} + \sqrt{2}g^Ty^{+})\\
&= \exp\left((x^{-}_L)^TJ({L, R})y^{-}_R + (x^{+}_L)^TJ({L, R})y^{+}_R  + \sqrt{2}h_{L}^Tx^+_{L} + \sqrt{2}g_{R}y^+_{R}\right) \\
&\quad \times \exp\left(x^-_uJ({\{u\},R})y^{-}_{R} + x^+_uJ({\{u\},R})y^{+}_{R} + \sqrt{2}h_u x^+_u\right) \\
&\quad \times \exp\left(x^-_vJ({\{v\},R})y^{-}_{R} + x^+_vJ({\{v\},R})y^{+}_{R} + \sqrt{2}h_v x^+_v\right)\\
&\quad \times \exp\left(x^-_{L}J({L,\{k\}})y^{-}_k -x^-_uJ_{uk}y^-_{k} - x^-_vJ_{vk}y^-_k + x^+J({V_{obs},\{k\}})y^{+}_k + \sqrt{2}g_k y^+_k\right)
\end{align*}
Here $x_T (y_T)$ denote the restriction of $x (y)$ to all indices in $T$ and similarly $J(T_1, T_2)$ denote the sub-matrix obtained by restricting $J$ to the rows and columns indexed by $T_1, T_2$ respectively. We can show that each of the last three terms in the product can be straightforwardly bounded in $[\exp(-2\lambda), \exp(2\lambda)]$. Observe that
\begin{align*}
&\exp\left(x^-_uJ(\{u\},R)y^{-}_{R} + x^+_uJ(\{u\},R)y^{+}_{R} + \sqrt{2}h_u x^+_u\right) \\
&= \exp\left(x_uJ(\{u\},R)y_{R} + x'_uJ(\{u\},R)y'_{R} + h_u (x_u + x'_u)\right)\\
& \ge \exp\left(-2\left(\sum_{j \in R} |J_{uj}| + |h_u|\right)\right) \ge \exp(-2\lambda)
\end{align*}
Similarly we can bound $\exp\left(x^-_vJ({\{v\},R})y^{-}_{R} + x^+_vJ(\{v\},R)y^{+}_{R} + \sqrt{2}h_v x^+_v\right) \ge \exp(-2\lambda)$. As for the last term, we have
\begin{align*}
&\exp\left(x^-_{L}J(L,\{k\})y^{-}_k -x^-_uJ_{uk}y^-_{k} - x^-_vJ_{vk}y^-_k + x^+J_{V_{obs},\{k\}}y^{+}_k + \sqrt{2}g_k y^+_k\right)\\
&= \exp\left(x_{L}J(L,\{k\})y_k + x'_{L}J(L,\{k\})y'_k +x_uJ_{uk}y'_{k} + x'_uJ_{uk}y_{k} + x_vJ_{vk}y'_k + x'_vJ_{vk}y_k  + g_k (y_k + y'_k)\right)\\
&\ge \exp\left(-2\left(\sum_{i \in V_{obs}} |J_{ik}| + |g_k|\right)\right) \ge \exp(-2\lambda)
\end{align*}
Now, setting 
\[
\rho(L,R) := \sum\limits_{\substack{x_L,x'_L \in \{\pm 1\}^{|L|}\\y_R,y'_R \in \{\pm 1\}^{|R|}}}\exp\left((x^{-}_L)^TJ(L,R)y^{-}_R + (x^{+}_L)^TJ(L,R)y^{+}_R  + \sqrt{2}h_{L}^Tx^+_{L} + \sqrt{2}g_{R}y^+_{R}\right),
\]
we have
\begin{align*}
N&\ge \alpha^2\exp(-6\lambda) \rho(L,R) \sum\limits_{x_u,x'_u \in \{\pm 1\}}(x^{-}_u)^{2}\sum\limits_{x_v,x'_v \in \{\pm 1\}}(x^{-}_v)^{2}\sum\limits_{y_k,y'_k \in \{\pm 1\}}(y^{-}_k)^{2}\\
&= 2^6\exp(-6\lambda) \rho(L,R).
\end{align*}
Here the second equality follows from observing that $\sum\limits_{x_u,x'_u \in \{\pm 1\}}(x^{-}_u)^{2} = 4$ (similarly for $x^-_v$ and $y^-_k$).
 Similarly, the denominator can be bounded as follows,
\begin{align*}
D &= \sum\limits_{\substack{x,x' \in \{\pm 1\}^n\\y,y' \in \{\pm 1\}^m}}\exp\left((x^{-}_L)^TJ(L,R)y^{-}_R + (x^{+}_L)^TJ(L,R)y^{+}_R  + \sqrt{2}h_{L}^Tx^+_{L} + \sqrt{2}g_{R}y^+_{R}\right) \\
&\quad \times \exp\left(x^-_uJ(\{u\},R)y^{-}_{R} + x^+_uJ(\{u\},R)y^{+}_{R} + \sqrt{2}h_u x^+_u\right) \\
&\quad \times \exp\left(x^-_vJ(\{v\},R)y^{-}_{R} + x^+_vJ(\{v\},R)y^{+}_{R} + \sqrt{2}h_v x^+_v\right)\\
&\quad \times \exp\left(x^-J_{V_{obs},\{k\}}y^{-}_k + x^+J_{V_{obs},\{k\}}y^{+}_k + \sqrt{2}g_k y^+_k\right)\\
&\le \exp(6\lambda)\sum\limits_{\substack{x,x' \in \{\pm 1\}^n\\y,y' \in \{\pm 1\}^m}}\exp\left((x^{-}_L)^TJ(L,R)y^{-}_R + (x^{+}_L)^TJ(L,R)y^{+}_R  + \sqrt{2}h_{L}^Tx^+_{L} + \sqrt{2}g_{R}y^+_{R}\right) \\
&= 2^6 \exp(6\lambda) \rho(L,R).
\end{align*}
Combining, we have $\cov(u,v) \ge \alpha^2\exp(-12 \lambda)$.
\end{proof}
\begin{corollary}
For $u \neq v \in V_{obs}$ such that there exists $w \in V_{lat}$ with $(u,k), (v,k) \in E$ and a subset of observed nodes $S \subseteq V_{obs} \backslash \{u,v\}$, $\cov^{\avg}(u,v| X_S) \ge {\alpha^2}\exp(-12\lambda)$.
\end{corollary}
\begin{proof}
Since for any $X_S = x_S$, by Lemma \ref{lem:pos}, the covariance is bounded below by ${\alpha^2}\exp(-12\lambda)$, hence the expectation is also bounded by the same quantity.
\end{proof}

\paragraph{Remark.} Observe that the above lemma also shows that $N_2(u) \subseteq N^{mkv}_2(u)$. It is not hard to see that $N^{mkv}_2(u) \subseteq N_2(u)$ by the structure of the RBM therefore $N_2(u) = N^{mkv}_2(u)$.

\paragraph{Remark.} The key structural result can be extended to the setting in which there are edges between hidden and observed variables using the same techniques, however now the bound will depend on the length of the shortest path connecting two observed nodes similar to \cite{bresler2018learning}.

\section{Algorithm}
In this section we present the main algorithm (Algorithm \ref{algo:main}) and a proof of its correctness.  Our algorithm and analysis is similar to the influence maximization algorithms for learning ising models as in \cite{bresler2015efficiently}. However, instead of maximizing influence, our algorithm exploits the key property to maximize conditional covariance. For completeness, we give the full proof. 
\begin{algorithm}\label{algo:main}
   \caption{$\textsc{LearnRBMNbhd}$ Learn 2-hop neighborhood of a node}
   \hspace*{\algorithmicindent} \textbf{Input} Samples $X^{(1)}, \ldots, X^{(M)}$, threshold $\tau$, observed node $u$ \\
    \hspace*{\algorithmicindent} \textbf{Output} Set $S$ of two-hop neighbors of $u$
\begin{algorithmic}[1]
  \STATE Set $S := \phi$
   \STATE Let $i^*, \eta^* = \argmax_v \hat{\cov}^{\avg}(u,v|S), \max_v \hat{\cov}^{\avg}(u,v|S)$
   \IF {$\eta^* \ge \tau$}
    \STATE $S = S \cup \{i^*\}$
  \ELSE 
  \STATE Go to Step 9
  \ENDIF
  \STATE Go to Step 2
   \STATE {\em Pruning step}: For each $v \in S$, if $\hat{\cov}^{\avg}(u, v| X_S) < \tau$, remove $v$
   \STATE Return $S$
\end{algorithmic}
\end{algorithm}
\begin{theorem}\label{thm:main}
Consider $M$ samples drawn from an $(\alpha, \lambda)$-locally consistent RBM, $X^{(1)}, \ldots, X^{(M)}$. For $\tau = \frac{\alpha^2}{2}\exp(-12 \lambda)$, with probability $1 - \zeta$, $\textsc{LearnRBMNbhd}(X^{(1)}, \ldots, X^{(M)}, \tau, u)$ outputs {\em exactly} the two-hop neighborhood of each observed variable $u$ as long as 
\[
M \ge \Omega\left(\left(\log(1/\zeta) + T^*\log(n)\right)\frac{2^{2T^*}}{\tau^2 \delta^{2T^*}}\right)\text{ for } T^* = \frac{8}{\tau^2}.
\]
Moreover, the algorithm runs in time $O(T^* M n)$ for each node $u$.
\end{theorem}
\begin{proof}
The proof follows along the same lines as \cite{bresler2015efficiently}. We will first show that our estimates of conditional covariance are close to the true values with the given $M$ samples. We will then show that after $T$ iterations, set $S$ contains a superset of the two-hop neighbors. Lastly we will show that our refining step removes all nodes except the two-hop neighbors. This will complete our proof.

\paragraph{Closeness of Estimates.} Denote by $\mathcal{A}(l,\epsilon)$ the event such that for all $u,v$ and $S$ with $|S| \le l$, simultaneously, $\left|\hat{\cov}^{\avg}(u,v|S) - \cov^{\avg}(u,v|S)\right| \le \epsilon$.
\begin{lemma}\label{lem:close}
For fixed $l, \epsilon, \zeta \ge 0$, if the number of samples is $\Omega\left(\left(\log(1/\zeta) + l\log(n)\right)\frac{2^{2l}}{\epsilon^2 \delta^{2l}}\right)$.
then $\Pr[A(l,\epsilon)] \geq 1 - \zeta$.
\end{lemma}
We defer the proof of the above lemma to the appendix. Choosing $M = \Omega\left(\left(\log(1/\zeta) + T^*\log(n)\right)\frac{2^{2T}}{\tau^2 \delta^{2l}}\right)$, we have $A:= A(T^*, \tau/2)$ holds for $T^* = 8/\tau^2$ with probability $1- \zeta$. From now om we assume $A$ holds.

\paragraph{Entropy Gain.} We will show that the conditional mutual information is bounded below by a function of the average conditional covariance thus at each iteration of the algorithm we are increasing the overall entropy of $X_u$.
\begin{lemma}
For $u \neq v \in V_{obs}$ and a subset of observed nodes $S \subseteq V_{obs} \backslash \{u,v\}$ with configuration $x_S$,
\[
\sqrt{2I(X_u;X_v|X_S)} \ge \cov^{\avg}(u,v|S)
\]
\end{lemma}
\begin{proof}
We have
\begin{align*}
\sqrt{2I(X_u;X_v|X_S)} &= \sqrt{\E_{x_S}[2I(X_u;X_v|X_S = x_S)]} \\
&\geq \E_{x_S}[\sqrt{2I(X_u;X_v|X_S = x_S)}] \\
&= \E_{x_S}[\sqrt{2D_{KL}(\Pr(X_u,X_v|X_S = x_S)||\Pr(X_u|X_S = x_S)\Pr(X_v|X_S = x_S))}]\\
&\geq 2 \E_{x_S}[D_{TV}(\Pr(X_u,X_v|X_S = x_S)||\Pr(X_u|X_S = x_S)\Pr(X_v|X_S = x_S))]\\
&= \E_{x_S}\left[\sum_{x_u,x_v \in \{\pm 1\}}\left|\Pr(X_u = x_u,X_v = x_v|X_S = x_S) \right.\right.\\
&\qquad\qquad\qquad \left.\left.- \Pr(X_u = x_u|X_S = x_S)\Pr(X_v = x_v|X_S = x_S)\right|\right]\\
&= \E_{x_S}\left[\sum_{x_u,x_v \in \{\pm 1\}}\left|x_ux_v\Pr(X_u = x_u,X_v = x_v|X_S = x_S) \right.\right.\\
&\qquad\qquad\qquad \left.\left.- x_u\Pr(X_u = x_u|X_S = x_S)x_v\Pr(X_v = x_v|X_S = x_S)\right|\right]\\
&\ge \E_{x_S}\left[\sum_{x_u,x_v \in \{\pm 1\}}(x_ux_v\Pr(X_u = x_u,X_v = x_v|X_S = x_S) \right.\\
&\qquad\qquad\qquad \left.- x_u\Pr(X_u = x_u|X_S = x_S)x_v\Pr(X_v = x_v|X_S = x_S))\right]\\
&= \E_{x_S}\left[\E[X_uX_v|X_S = x_S] - \E[X_u|X_S = x_S]\E[X_v|X_S = x_S]\right]\\
&= \E_{x_S}\left[\cov(u,v|X_S = x_S)\right] = \cov^{\avg}(u,v|S).
\end{align*}
Here the first inequality follows using Jensen's and the second inequality follows from the Pinsker's inequality and the rest follow from simple algebraic manipulations.
\end{proof}

\paragraph{Upper Bound on Size of $S$.} We will show that $|S| \le T^*$. Let the sequence of added nodes be $i_1, \ldots, i_T$ for some $T$ and $S_l = \{i_1, \ldots, i_l\}$ for $1 \le l \le T$. For each $j \in T$, we have $\hat{\cov}^{\avg}(u;i_j|X_{S_j}) \ge \tau$ (by Step 3). If $T \ge T^*$, then we have $\cov^\avg(u;i_j|X_{S_j}) \ge \tau/2$ for all $j \le T^* + 1$ (since $A$ holds). Thus we have,
\[
1 \ge H(X_u) \ge I(X_u|X_S) = \sum_{j=1}^TI(X_u;X_{i_j}|S_{j-1}) \ge \frac{T^* + 1}{8}\tau^2.
\]
Here the inequalities follow from standard properties of entropy and mutual information. This leads to a contradiction since $T^* = \frac{8}{\tau^2}$. Thus, we have $T \le T^*$. Observe that each iteration requires $O(Mn)$ time and at most $T^*$ iterations take place prior to pruning. Also pruning takes $O(Mn)$ time, giving us a total runtime of $O(T^*Mn)$.

\paragraph{Recovery of Two-hop Neighborhood.}  We will show that $N_2(u) \subseteq S$. Suppose $N_2(u) \not \subseteq S$, then there exists $v \in N_2(u)$. By Lemma \ref{lem:cov}, we know that $\cov^\avg(u,v|X_S) \ge {\alpha^2}\exp(-12\lambda) = 2 \tau$. Since $A$ holds and $|S| \le 8/\tau^2$, we have $\hat{\cov}^\avg(u,v|X_S) \ge 3\tau/2$, thus the algorithm would not have terminated. This is a contradiction, thus $N_2(u) \subseteq S$ before pruning.

Now if $v \not\in N_u(S)$ then $\cov(u, v|X_{S \backslash \{v\}}) = 0$ since conditional on the 2-hop neighborhood, $X_u$ and $X_v$ are independent, therefore they will be removed. Whereas, by Lemma \ref{lem:cov}, if $v \in N_u(S)$ then $\cov(u, v|X_{S \backslash \{v\}}) \ge 2\tau$ and our test will not remove it (estimates of covariance are correct withing $\alpha/2$). Thus we will exactly obtain the neighborhood at the end of the algorithm.
\end{proof}

\section{Hardness of Learning General RBMs}
In this section we will discuss why our model does not violate the hardness result stated in \cite{bresler2018learning}. The hardness result in the paper reduces the problem of learning sparse parities with noise over the uniform distribution to the problem of structure recovery of a RBM. 

Suppose $S \subseteq [n]$ is the subset on which the parity problem is defined. The main technique used for the reduction is the observation from \cite{klivans2017learning} that the joint distribution on the input and noisy parity $(x,y)$ can be represented as a single term MRF (term $y\prod_{i \in S}x_i$). Further \cite{bresler2018learning} showed that every MRF can be represented as a RBM with sufficiently many hidden units. Here we show that even if the external fields are arbitrary, any ferromagnetic RBM when expressed as an MRF has pairwise potentials for every two-hop neighbor pair. This implies that it cannot represent the MRF corresponding to the noisy parity.

\begin{lemma}[\cite{bresler2018learning}]
Given a RBM, with $\rho(a) = \log(\exp(a) + \exp(-a))$, we have,
\[Pr[X= x] =  \frac{1}{Z} \exp\left(\sum_{j=1}^m\rho(x^TJ({V_{obs}, \{j\}}) + g_j) + h^Tx\right).\] 
\end{lemma}
Let us look at the potential corresponding to $k \in V_{lat}$, $\rho(x^TJ_{V_{obs}, \{k\}} + g_k)$. We will show that when you expand the term over the monomial basis, the coefficient corresponding to $x_ix_j$ for any $i,j \in V_{obs}$ is non-negative and the coefficient corresponding to $x_ux_v$ for $u,v \in V_{obs}$ such that $k \in N(u) \cap N(v)$ is strictly positive. More formally,
\begin{lemma} \label{lem:parity}
$f(x) = \sum_{j=1}^m\rho(x^TJ({V_{obs}, \{j\}}) + g_j) + h^Tx$ when expressed in the monomial basis with coefficients $\hat{f}_S$ for every monomial $S$ satisfies: $\hat{f}_{\{i,j\}} \ge 0$ for all $i,j \in V_{obs}$, moreover, $\hat{f}_{\{i,j\}} > 0$ for $i, j$ such that $i \in N_2(j)$.
\end{lemma}
We defer the proof of the above lemma to the appendix. Since we sum such potentials, this positive coefficient cannot be canceled and $f$ cannot represent the parity MRF as in the reduction. This raises the question of understanding the exact class of RBMs for which the hardness results truly holds.

\section{Conclusions and Open Problems}
In this work we presented a key structural property of locally consistent RBMs with arbitrary external fields and subsequently showed how to use this property to iteratively build the two-hop neighborhood of each node. Our algorithm runs in optimal time and sample complexity in terms of the dimension however pays doubly exponentially in the upper bound on the weights. This seems to be an artifact of the approach of maximizing influence in general whereas algorithms using convex optimization are able to avoid this dependence for fully-observed graphical models. A natural open question is to improve this dependency potentially using tools from convex optimization. Alternatively, proving a stronger structural result such as weak-submodularity could lead to the currect dependency. More broadly, understanding the most expressive class of RBMs that allow efficient structure learning while not violating the hardness result is a worthwhile future direction to pursue.

\paragraph*{Acknowledgements.} The author would like to thank Sumegha Garg and Jessica Hoffmann for comments on the initial draft, and Adam Klivans, Frederic Koehler and Josh Vekhter for useful discussions. 

\bibliographystyle{plain}
\bibliography{references}
\appendix

\section{Omitted Proofs}
\begin{proof}[Proof of Lemma \ref{lem:close}]
The proof follows essentially from \cite{bresler2015efficiently}. Let $m$ denote the number of samples. Using standard concentration inequalities, we know that for any subset $W \subseteq V_{obs}$ and configuration $x_W \in \{\pm\}^{|W|}$, we have
\[
\Pr(|\hat{\Pr}(X_W = x_W) - \Pr(X_W - x_W)| \ge \gamma) \le 2 \exp(-2\gamma^2m).
\]
We need the above to hold over all possible choices of $W$ and $x_W$ with $|W| \le l+2$. There are at most $\sum_{k = 1}^{l+2}2^k{n \choose k} \le (l+2)(2n)^{l+2}$ many choices. Thus for $m \ge \frac{\log(2(l+2)) + \log(1/\zeta) + (l+2)\log(2n)}{2 \gamma^2}$, with probability, $1 - \zeta$, for all $W$ and $x_W$ with $|W| \le l+2$, we have $|\hat{\Pr}(X_W = x_W) - \Pr(X_W - x_W)| \le \gamma$.

Now assuming that the above is true, we will show that $\left|\hat{\cov}^{\avg}(u,v|S) - \cov^{\avg}(u,v|S)\right|$ is bounded for all $|S| \le l$. We have
\begin{align*}
 &\left|\hat{\cov}^{\avg}(u,v|S) - \cov^{\avg}(u,v|S)\right| \\
 &=
 \left|\hat{\E}_{x_S}\left[\hat{\E}[X_uX_v|X_S = x_s] - \hat{\E}[X_u|X_S = x_S]\hat{\E}[X_v| X_s = x_S]\right] \right.\\
 &\qquad \left.- \E_{x_S}\left[\E[X_uX_v|X_S = x_s] - \E[X_u|X_S = x_S]\E[X_v| X_s = x_S]\right]\right|\\
  &=\left|\sum_{x_u,x_v}x_ux_v\left(\hat{\E}_{x_S}\left[\hat{\Pr}[X_u = x_u, X_v = x_v|X_S = x_s] - \hat{\Pr}[X_u = x_u|X_S = x_S]\hat{\Pr}[X_v = x_v| X_s = x_S]\right] \right.\right.\\
 &\qquad \left.\left.- \E_{x_S}\left[\Pr[X_u = x_u, X_v = x_v|X_S = x_s] - \Pr[X_u = x_u|X_S = x_S]\Pr[X_v = x_v| X_s = x_S]\right]\right)\right|\\
 &\le \sum_{x_u,x_v, x_S}\left|\left[\hat{\Pr}[X_u = x_u, X_v = x_v, X_S = x_s] - \hat{\Pr}[X_u = x_u, X_S = x_S]\hat{\Pr}[X_v = x_v| X_s = x_S]\right] \right.\\
 &\qquad \left.- \Pr[X_u = x_u, X_v = x_v, X_S = x_s] - \Pr[X_u = x_u, X_S = x_S]\Pr[X_v = x_v| X_s = x_S]\right|\\
 &= 2^{|S|+2} \gamma + \sum_{x_u,x_v, x_S}\left|\hat{\Pr}[X_u = x_u, X_S = x_S]\hat{\Pr}[X_v = x_v| X_s = x_S] - \Pr[X_u = x_u, X_S = x_S]\Pr[X_v = x_v| X_s = x_S]\right|
\end{align*}
The second term can be bounded as follows,
\begin{align*}
&\left|\hat{\Pr}[X_u = x_u, X_S = x_S]\hat{\Pr}[X_v = x_v| X_s = x_S] - \Pr[X_u = x_u, X_S = x_S]\Pr[X_v = x_v| X_s = x_S]\right|\\
& \le \hat{\Pr}[X_v = x_v| X_s = x_S]\left|\hat{\Pr}[X_u = x_u, X_S = x_S] - \Pr[X_u = x_u, X_S = x_S]\right| \\
&\qquad + \Pr[X_u = x_u, X_S = x_S]\left|\hat{\Pr}[X_v = x_v| X_s = x_S] - \Pr[X_v = x_v| X_s = x_S]\right|\\
&\le \gamma + \left|\frac{\hat{\Pr}[X_v = x_v, X_s = x_S]}{\hat{\Pr}[X_S = x_S]} - \frac{\Pr[X_v = x_v, X_s = x_S]}{\Pr[X_S = x_S]}\right|\\
&\le \gamma + \left|\frac{\hat{\Pr}[X_v = x_v, X_s = x_S]}{\hat{\Pr}[X_S = x_S]} - \frac{\Pr[X_v = x_v, X_s = x_S]}{\hat{\Pr}[X_S = x_S]}\right| + \left|\frac{\Pr[X_v = x_v, X_s = x_S]}{\hat{\Pr}[X_S = x_S]} - \frac{\Pr[X_v = x_v, X_s = x_S]}{\Pr[X_S = x_S]}\right|\\
&\le \gamma + \frac{\gamma}{\delta^{|S|} - \gamma} + \frac{\gamma}{\delta^{|S|}}.
\end{align*}
Choosing $\gamma \le \epsilon 2^{-l}\frac{\delta^l}{20}$, we get,
\begin{align*}
\left|\hat{\cov}^{\avg}(u,v|S) - \cov^{\avg}(u,v|S)\right| &\le 2^{|S|+2}\left(2\gamma +  \frac{\gamma}{\delta^{|S|} - \gamma} + \frac{\gamma}{\delta^{|S|}} \right) \le \epsilon.
\end{align*}
Thus, we have $\Pr[A(l,\epsilon)] \ge 1 - \zeta$ for $m = \Omega\left(\left(\log(1/\zeta) + l\log(n)\right)\frac{2^{2l}}{\epsilon^2 \delta^{2l}}\right)$
\end{proof}

\begin{proof}[Proof of Lemma \ref{lem:parity}]
 Let $c_ij$ be the coefficient corresponding to $x_ix_j$ and $L = V_{obs} \backslash \{i,j\}$, then using standard fourier expansion, we have
\begin{align*}
c_{ij} &= \sum_{x \in \{\pm 1\}^n}\rho(x^TJ_{V_{obs}, \{k\}} + g_k)x_ix_j\\
&= \sum_{x_{L} \in \{\pm 1\}^{|L|}}\left(\rho(x^TJ_{L, \{k\}} + J_{ik} + J_{jk} + g_k) + \rho(x^TJ_{L, \{k\}} - J_{ik} - J_{jk} + g_k) \right.\\
&\qquad \qquad\left.- \rho(x^TJ_{L, \{k\}} - J_{ik} + J_{jk} + g_k) - \rho(x^TJ_{L, \{k\}} + J_{ik} - J_{jk} + g_k)\right)\\
&= \sum_{x_{L} \in \{\pm 1\}^{|L|}} \log\left(\frac{\exp(2x^TJ_{L, \{k\}} + 2g_k) + \exp(-2x^TJ_{L, \{k\}} - 2g_k) + \exp(2J_{ik} + 2J_{jk}) + \exp(-2J_{ik} - 2J_{jk})}{\exp(2x^TJ_{L, \{k\}} + 2g_k) + \exp(-2x^TJ_{L, \{k\}} - 2g_k) + \exp(2J_{ik} - 2J_{jk}) + \exp(-2J_{ik} + 2J_{jk})}\right)
\end{align*}
Observe that $\exp(a) + \exp(-a)$ is an increasing function of $|a|$, since $|J_{ik} + J_{jk}| \ge |J_{ik} - J_{jk}|$ ($J \ge 0$), therefore $\exp(2J_{ik} + 2J_{jk}) + \exp(-2J_{ik} - 2J_{jk}) \ge \exp(2J_{ik} - 2J_{jk}) + \exp(-2J_{ik} + 2J_{jk})$. Thus the above term in non-negative. Also notice if $J_{ik}, J_{jk} > 0$ then the sum is strictly greater than 0. Thus we have the desired property for $c_{ij}$.
\end{proof}
\end{document}